\documentclass[conference]{IEEEtran}

\usepackage[table]{xcolor}
\usepackage{cite}
\usepackage{amsmath,amssymb,amsfonts}
\usepackage{dsfont}
\usepackage{algorithmic}
\usepackage{graphicx}
\usepackage{textcomp}
\usepackage{xcolor}
\def\BibTeX{{\rm B\kern-.05em{\sc i\kern-.025em b}\kern-.08em
    T\kern-.1667em\lower.7ex\hbox{E}\kern-.125emX}}
\usepackage{algorithm}
\usepackage{algorithmic}
\usepackage{hyperref}

\def\<{\langle}
\def\>{\rangle}

\newcommand{\normal}{\mathcal{N}}

\newcommand{\beq}{\begin{equation}}
\newcommand{\eeq}{\end{equation}}

\newcommand{\sign}{\textrm{\sign}}
\newcommand{\E}{\mathbb{E}}

\newcommand{\mb}{\mathbb}
\newcommand{\mc}{\mathcal}

\newtheorem{thm}{Theorem}[section] 
\newtheorem{lem}[thm]{Lemma} 
\newenvironment{proof}{\textit{Proof:}}{\hfill$\square$}

\begin{document}

\title{Flexible and Efficient \\ Drift Detection without Labels
}


\author{
\IEEEauthorblockN{Nelvin Tan}
\IEEEauthorblockA{
\textit{American Express}\\
thongcainelvin.tan@aexp.com}
\and
\IEEEauthorblockN{Yu-Ching Shih}
\IEEEauthorblockA{
\textit{American Express}\\
yuching.shih1@aexp.com}
\and
\IEEEauthorblockN{Dong Yang}
\IEEEauthorblockA{
\textit{American Express}\\
dong.yang@aexp.com}
\and
\IEEEauthorblockN{Amol Salunkhe}
\IEEEauthorblockA{
\textit{American Express}\\
amol.salunkhe@aexp.com}
}

\maketitle

\begin{abstract}
    Machine learning models are being increasingly used to automate decisions in almost every domain, and ensuring the performance of these models is crucial for ensuring high quality machine learning enabled services. Ensuring concept drift is detected early is thus of the highest importance. A lot of research on concept drift has focused on the supervised case that assumes the true labels of supervised tasks are available immediately after making predictions. Controlling for false positives while monitoring the performance of predictive models used to make inference from extremely large datasets periodically, where the true labels are not instantly available, becomes extremely challenging. We propose a flexible and efficient concept drift detection algorithm that uses classical statistical process control in a label-less setting to accurately detect concept drifts. We shown empirically that under computational constraints, our approach has better statistical power than previous known methods. Furthermore, we introduce a new semi-supervised drift detection framework to model the scenario of detecting drift (without labels) given prior detections, and show our how our drift detection algorithm can be incorporated effectively into this framework. We demonstrate promising performance via numerical simulations.
\end{abstract}


\section{Introduction}

Machine Learning has been adopted as a popular technique for providing data-driven prediction capabilities, and is at the core of several otherwise complicated or intractable tasks. The ability to generalize and extrapolate from data has made its usage attractive as a general approach to data driven problem solving. However, the generalizability of models relies on an important assumption of data \textit{stationarity}, which states that the data points in the data stream should be identically and independently distributed (i.i.d.), i.e., datapoints are generated from the same distribution \cite{liobait2010}. While this is a reasonable assumption, in practice it is often violated. The data that models work with in production might become significantly different than the static dataset that the model was trained on -- a phenomenon named \textit{concept drift} \cite{gama2014,wang2015}.

We consider the \textit{abrupt} drift setting, where a new concept occurs in a short period of time (e.g. at the beginning of COVID-19 in March 2020, stock prices suddenly changed). Following the parlace of \cite{Gemaque2020}, we focus on non-online drift detection setting with whole-window detection, where the entire window is used in the detection. Most importantly, we consider the label-less (i.e., unsupervised) setting where labels of the new data (i.e., during production) are unavailable.

\subsection{Problem Setup} \label{sec:setup}

Formally, a concept drift is the problem where there exists a time $t$ such that
\begin{align}
    P_t(\boldsymbol{x},y)\neq P_{t+1}(\boldsymbol{x},y),
\end{align}
where $P_t$ denotes the joint probability distribution of the features $\boldsymbol{x}\in\mb{R}^m$ and target $y\in\mb{R}$. By the chain rule of probability, we have
\begin{align}
    P_t(\boldsymbol{x},y)
    =P_t(\boldsymbol{x})\cdot P_t(y|\boldsymbol{x}),
\end{align}
which implies the following types of drift \cite{Lu2028}:
\begin{itemize}
    \item \textbf{Virtual drift.} We have $P_t(\boldsymbol{x})\neq P_{t+1}(\boldsymbol{x})$ while $P_t(y|\boldsymbol{x})=P_{t+1}(y|\boldsymbol{x})$.
    \item \textbf{Real drift.} We have $P_t(y|\boldsymbol{x})\neq P_{t+1}(y|\boldsymbol{x})$ while $P_t(\boldsymbol{x})=P_{t+1}(\boldsymbol{x})$.
    \item \textbf{Mixture drift.} Mixture of virtual drift and actual drift, namely, $P_t(y|\boldsymbol{x})\neq P_{t+1}(y|\boldsymbol{x})$ and $P_t(\boldsymbol{x})\neq P_{t+1}(\boldsymbol{x})$.
\end{itemize}
Without access to the labels $y$, we can only detect virtual drift and mixture drift. Note that virtual drift may still cause problems for the inference model, since changes in $P_t(y|\boldsymbol{x})$ may change the error of the inference model, even if $P_t(y|\boldsymbol{x})=P_{t+1}(y|\boldsymbol{x})$ \cite{Hoens2012}. In this paper, we mainly look at the following problem setup:

\textbf{Drift detection without labels (problem 1).} Given a datastream, we have two windows, the \textit{reference} window/set $\mc{S}_R$ and the \textit{detection} window/set $\mc{S}_D$. Furthermore, we have access to the \textit{training} set $\mc{S}_T$ that was used to train the inference model and is of the same distribution as $\mc{S}_R$. The datapoints in the reference window are labeled whereas the datapoints in the detection window are unlabeled and may (or may not) have suffered a concept drift. Let us denote the distribution of $\mc{S}_T$ and $\mc{S}_R$ as $P_R$ and the distribution of $\mc{S}_D$ as $P_D$. The goal is to detect whether $P_R(\boldsymbol{x})=P_D(\boldsymbol{x})$.

\textbf{Drift detection with prior detections (problem 2).} We also introduce a new distinct semi-supervised setting for drift detection: In this setting, we assume that prior detections have already taken place. Specifically, we have access to prior training, reference, and detection data that has been used for prior drift detection. Furthermore, we also have access to the result of the drift detection. Mathematically, we say that given
\begin{align}
    \big\{\mc{S}_T^{(i)},\mc{S}_R^{(i)},\mc{S}_D^{(i)},z^{(i)}\big\}_{i=1}^M\quad
    \text{where $i\in\{1,\dots,M\}$},
    \label{eq:prior knowledge}
\end{align}
where there are $M$ prior drift detections and we used $\mc{S}_T^{(i)},\mc{S}_R^{(i)},\mc{S}_D^{(i)}$ for the $i$th drift detection to identify the outcome $z^{(i)}\in\{0,1\}$, where $1$ signifies that a drift has taken place and $0$ signifies otherwise. With prior knowledge \eqref{eq:prior knowledge}, the goal is to detect a drift given a new triplet $\{\mc{S}_T,\mc{S}_R,S_D\}$ (as explained in the point above).

\textbf{Remark.} The motivation behind setup 2 is the scenario where a company has done many drift detections in the past using labeled data, which leads to high accuracy (which motivates the assumption of knowing $z^{(i)}$) of detection but is unsustainable. Therefore, at some point of time, they decide to switch to drift detection with unlablled data. To the best of our knowledge, this framework has been applied to several scenarios \cite{Lopez2017} but not in concept drift detection.

\textbf{Big data regime.} We consider the big data regime where there is a huge volume of data (e.g., finance). In this regime, the runtime of a drift detection approach becomes very important. We seek an approach that allows us to control the false positive rate, but at the sametime is statistically powerful and runtime efficient.

\textbf{Notation.} Throughout the paper, the function $\log(\cdot)$ has base $e$, and we make use of Bachmann-Landau asymptotic notation (i.e., $O$). $\boldsymbol{0}_n$ denotes a vector of zeros of length $n$, $\boldsymbol{1}_n$ denotes a vector of ones of length $n$, and $\boldsymbol{I}_n$ denotes the $n\times n$ identity matrix.

\subsection{Related Work}

\paragraph{Detection with labels} Popular methods for detecting concept drift rely on updating models as new data becomes available and when changes in data are detected. These methods assume the availability of labeled data \cite{gama2004,bifet2007,baena2006,goncalves2014,gama2014} -- we refer the interested reader to this survey \cite{Lu2028}. 

\paragraph{Detection without labels} While the use of labeled data for retraining and updating models is largely unavoidable, its use for the purpose of drift detection might be impractical as labeling is time consuming, expensive and in some cases, not a possibility at all \cite{lughofer2016,krempl2014}. Moreover the need for constant validation of the learned model leads to wasted labels, which are discarded when the model is found to be stable \cite{sethi2015}. This has motivated the development of {\em unlabeled} or {\em label-less} drift detection techniques which monitor changes to the feature distribution, as an indicator of drift. To stick to the theme of non-online and whole-window drift detection without labels, we will review related work in this area below -- we refer readers who are interested in other variants of the unsupervised drift detection to this survey \cite{Gemaque2020}. 

The most relevant work is that of \cite{Goldenberg2019}, which surveyed and analyzed the suitability of various distance measures for quantifying concept drift. Their key findings are:
\begin{itemize}
    \item Wasserstein distance and maximum mean discrepancy (MMD) scale well to high feature dimensions, but are impractical for large volumes of data (i.e. large sample sizes) due to high runtime complexity that depends on the size of the data.
    \item Hellinger distance, Kullback–Leibler divergence, Kolmogorov–Smirnov distance, and total variation distance can be numerically approximated for univariate data of any size. However, they do not scale up well to higher feature dimensions. Kolmogorov–Smirnov distance cannot be extended beyond bivariate case, and the rest (Hellinger distance, Kullback–Leibler divergence, and total variation distance) are limited by the exponential complexity of frequency/density calculation. Hellinger distance and Kullback–Leibler divergence have a closed-form solution for the multivariate Gaussian, but not for general multivariate distributions.
\end{itemize}
Overall, they recommended the Hellinger distance for drift detection in univariate or low-dimensional data. With the goal of improving efficiency for MMD, \cite[Section 6]{Gretton2012} calculated MMD over pair of datapoints (instead of taking MMD between the two sets), then obtained their test statistic by taking the averaging the MMD -- this has asymptotic Gaussian null distribution via the central liimt theorem. While this is very efficient, the test statistic has very high variance. To lower the variance, while still while retaining an asymptotically Gaussian null distribution, \cite{Zaremba2013} introduce the idea of batching, where the algorithm calculates the MMD over batches of datapoints, and the obtain the test statistic by averaging the MMD. \cite{Domingo2023} followed this up by showing that specifically for the MMD with permutation testing and sub-exponential data distributions, there exists a sample compression technique that is able to improve efficiency while also maintaining power. Our batching approach works for high-dimensional data and any distance measure (not just for MMD), makes no distribution assumption, and is runtime efficient. Furthermore, our method is fully unsupervised, model independent, task independent (allows for classification or regression).

Other related works on unsupervised drift detection include: \cite{Bashir2017} studied unlabeled drift detection for classification-type inference models, \cite{Maletzke2018} studied drift detection with target class imbalance, \cite{sethi2015} developed a drift detection that works with the SVM classifier and further improved it in \cite{sethi2017} to work with an ensemble of classifiers, \cite{Liu2018} focused on detecting drifts caused regional density changes, \cite{Li2019} studied unlabeled sequence anomaly in multi-dimensional sequences, and \cite{zheng2019,lundberg2017} studied the use of Shapley values, which are obtained by processing the features, to detect and also explain drift detection.

\subsection{Main Contributions}

Our main contributions are as follows:
\begin{itemize}
    \item We introduce a flexible and efficient framework (Algorithm \ref{alg:BD}) that uses the idea of batching for drift detection without labels -- the flexible choice being the option to choose the statistical distance between batches of data. We rigorously analyze the runtime complexity and show that it is more efficient compared to its non-batching counter parts, i.e., doing permutation testing on the distances obtained from the entire window/set. Numerical simulations are conducted to shown that batching is statistically more powerful than not batching, given a computational/runtime constraint.
    \item We introduce a distinct semi-supervised framework for the case where the analyst has access to information about prior detections (see Section \ref{sec:setup}). To this end, we provide an algorithm (Algorithm \ref{alg:classifier}) for this problem, that builds upon our batching approach, and analyze its performance via simulations.
\end{itemize}

\section{Algorithm for Problem 1}

Our framework and algorithm is presented in Algorithm \ref{alg:BD}, and a visual representation is presented in Figure \ref{fig:BD_algo}. Some remarks are due:
\begin{itemize}
    \item \textbf{Non parametric and parallelizable.} Our algorithm does not make any distribution assumption on the data provided. Our algorithm can be parallelized by computing the distances between batches in parallel.
    \item \textbf{Distance metric.} Our algorithm is flexible in the sense that we can choose the distance metric to use -- some examples are the Wasserstein-1 distance (also known as earth mover's distance (EMD)), maximum mean discrepancy (MMD), and Kullback–Leibler (KL) divergence.
    \item \textbf{Paired test.} We use the paired t-test instead of the 2-sample t-test because the two samples $\{d_{TR_1},\dots,d_{TR_n}\}$ and $\{d_{TD_1},\dots,d_{TD_n}\}$ are dependent of each other, through the involvement of the training set $\mathcal{S}_T$ in their calculation.
    \item \textbf{Use of training set.} One might point out that we could just compute batched distances between the (i) union of the training and reference sets, and (ii) the detection set, and then follow up with a one-sample t-test on the batched distances. This approach however is not as general as ours since certain distances, when computed, are always positive even for samples drawn from the same distribution (e.g., the computation of EMD \cite{Peyre2019}).
\end{itemize}

\begin{algorithm}
    \caption{Batched distance (BD) algorithm} \label{alg:BD}
    \begin{algorithmic}
        \STATE \textbf{input:} training set $\mathcal{S}_T$, reference set $\mathcal{S}_R$, detection set $\mathcal{S}_D$, significant level $\alpha$
        \STATE Partition training set into $n$ batches of size $k$, denoted as $\{\mathcal{S}_{T_1},\dots,\mathcal{S}_{T_n}\}$.
        \STATE Partition reference set into $n$ batches of size $k$, denoted as $\{\mathcal{S}_{R_1},\dots,\mathcal{S}_{R_n}\}$.
        \STATE Partition detection set into $n$ batches of size $k$, denoted as $\{\mathcal{S}_{D_1},\dots,\mathcal{S}_{D_n}\}$.
        \STATE Pair the training set batches with the reference window batches and the detection window batches in the following manner: $(\mathcal{S}_{T_1},\mathcal{S}_{R_1}),\dots,(\mathcal{S}_{T_n},\mathcal{S}_{R_n})$ and $(\mathcal{S}_{T_1},\mathcal{S}_{D_1}),\dots,(\mathcal{S}_{T_n},\mathcal{S}_{D_n})$.
        \FOR{each pair}
            \STATE Compute a distance metric between the samples, denoted as $d_{TR_1}=\text{dist}(\mathcal{S}_{T_1},\mathcal{S}_{R_1}),\dots,d_{TR_n}=\text{dist}(\mathcal{S}_{T_n},\mathcal{S}_{R_n})$ and $d_{TD_1}=\text{dist}(\mathcal{S}_{T_1},\mathcal{S}_{D_1}),\dots,d_{TD_n}=\text{dist}(\mathcal{S}_{T_n},\mathcal{S}_{D_n})$
        \ENDFOR
        \STATE Conduct a paired two-tailed hypothesis t-test: Denote $\mu$ as the true mean of $d_{TR_1}-d_{TD_1},\dots,d_{TR_n}-d_{TD_n}$. Then we have: $H_0:\mu=0,\,H_1:\mu\neq0$.
        \STATE Drift is detected if the p-value of the t-test is below $\alpha$.
    \end{algorithmic}
\end{algorithm}

\begin{figure}[t]
    \centering
    \includegraphics[width=1.0\columnwidth]{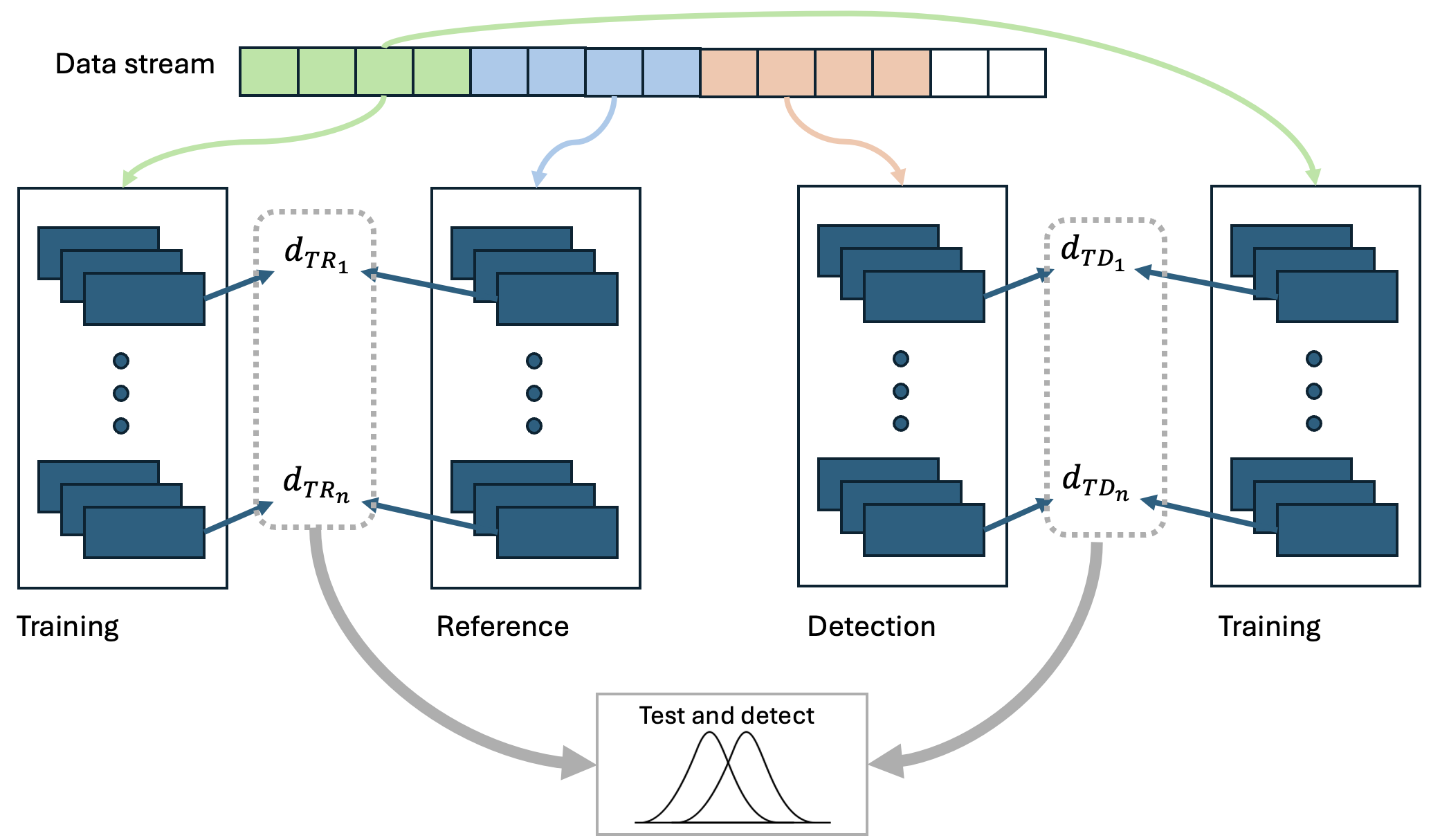}
    \caption{Visualization of Algorithm \ref{alg:BD}.}
    \label{fig:BD_algo}
\end{figure}

\paragraph{Assumptions} Our algorithm assumptions follow from those of the one-sample t-test.
\begin{itemize}
    \item \textbf{Independence.} The differences in distances $d_{TR_1}-d_{TD_1},\dots,d_{TR_n}-d_{TD_n}$ have to be independent of each other. This requires the datapoints in the batched triplets $(\mathcal{S}_{T_1},\mathcal{S}_{R_1},\mathcal{S}_{D_1}),\dots,(\mathcal{S}_{T_n},\mathcal{S}_{R_n},\mathcal{S}_{D_n})$ to be independent of each other.
    \item \textbf{Normality.} The difference of distance metrics $\big\{d_{TR_i}-d_{TD_i}\big\}_{i=1}^n$ have to follow the Gaussian distribution. Due to the robustness of the t-test, this does not matter too much when the number of batches $n$ is large since the mean of differences of distances approaches the Gaussian distribution via the central limit theorem.
    \item \textbf{Equally sized windows.} Requiring $n$ batches of size $k$ from $\mathcal{S}_T$, $\mathcal{S}_R$, and $\mathcal{S}_D$ imposes a constraint that all sets need to be of equal sizes. This is not particularly an issue since we can always generalize the algorithm to account for unequal sizes, this can be done by keeping the number of batches as $n$ and dropping the requirement for the batch size $k$. This results in paired batches of unequal sizes -- if the distance measures can still be calculated, then our algorithm would still work.
\end{itemize}

\paragraph{Runtime complexity} Given two samples of size $N$ and feature dimension $m$, let us define the runtime complexity of computing the distance metric as $d(N,m)$. By observing the $n$ iterations of distance computation in Algorithm \ref{alg:BD}, we have the following runtime result.

\begin{lem} \label{lem:BD_runtime}
The batched distance algorithm in Algorithm \ref{alg:BD} has a runtime complexity of $O\big(n\cdot d(k,m)\big)$, whereas a permuation tests with $B$ permutations (and without batching) gives a runtime complexity of $O\big(B\cdot d(nk,m)\big)$.
\end{lem}

Essentially, as long as $d(N,m)$ is linear in $m$ and polynomial in $N$, then batching will always have a lower runtime complexity compared to permutation testing without batching. We illustrate through some distance metric choices how batching reduces the runtime complexity of drift detection. 

\textbf{Choice of distances and motivation.} We have selected 3 distance metrics: (i) Earth mover's distance, (ii) Maximum mean discrepancy, and (iii) Kullback-Leibler divergence. We have included the earth mover's distance and maximum mean discrepancy because they are known to work better for higher dimensions \cite{Goldenberg2019}. Note that our list of distance measures is non-exhaustive -- see \cite{Goldenberg2019} for a survey of distance measures that can be used for drift detection.

\textbf{Earth mover's distance.} More generally, the integral probability metric between two distributions $P$ and $Q$ is defined as
\begin{align}
    d_{\mathcal{F}}(P,Q)
    =\sup_{f\in\mathcal{F}}\Big|\E_{X\sim P}[f(X)]-E_{Y\sim Q}[f(Y)]\Big|,
    \label{eq:IPM}
\end{align}
where $\mathcal{F}$ is a class of functions to be chosen and $f:\mb{R}^m\rightarrow\mb{R}$. Choosing $\mathcal{F}$ to be the class of 1-Lipschitz functions, i.e.,
\begin{align*}
    \mathcal{F}=\{f:f\text{ continuous},\,|f(\boldsymbol{x})-f(\boldsymbol{y})|\leq\|\boldsymbol{x}-\boldsymbol{y}\|\},
\end{align*}
where $\|\cdot\|$ is a norm, gives us the EMD. The EMD can be estimated but viewing it as an optimal transport problem and the solving it with linear programming \cite{Peyre2019}.

\begin{lem} \label{lem:EMD}
    Using the batched distance algorithm in Algorithm \ref{alg:BD}, the runtime complexity of using the earth mover's distance (EMD) drops from $O(B(n^3k^3 + mn^2k^2))$ to $O(n(k^3+mk^2))$, where $B$ is the number of iterations in permutation testing.
\end{lem}

\begin{proof}
    It is known that for 2 samples of size $N$, EMD can be solved in $d(N,m)=O(N^3 + mN^2)$ time \cite{Peyre2019}.
    \begin{itemize}
    \item Without batching, computing the p-value for EMD via permutation testing would require permutation which requires $B$ additional iterations for each distance computed. As a result, computing the p-value for EMD via permutation testing would take $O(B(N^3 + mN^2))=O(B(n^3k^3 + mn^2k^2))$, where we set $N=nk$.
    \item For the BD algorithm, setting $d(k,m)=O(k^3\log k)$ in Lemma \ref{lem:BD_runtime} gives us the runtime complexity.
    \end{itemize} 
\end{proof}

\textbf{Maximum mean discrepancy.} Choosing $\mathcal{F}$ in \eqref{eq:IPM} to be the unit ball in a reproducing kernel Hilbert space gives us the MMD. We choose our kernel function to be the radial basis function kernel
\begin{align*}
    k(\boldsymbol{x},\boldsymbol{y})=\exp\Big(-\frac{\|\boldsymbol{x}-\boldsymbol{y}\|_2^2}{2\sigma^2}\Big).
\end{align*}
Given a sample size of $n_*$, the MMD can be estimated by the U-statistic \cite{Gretton2012}:
\begin{align}
    &\frac{1}{n_*(n_*-1)}\sum_{i=1}^n\sum_{j\neq i}k(\boldsymbol{x}_i,\boldsymbol{x}_j) \nonumber \\
    &\quad\quad+\frac{1}{n_*(n_*-1)}\sum_{i=1}^n\sum_{j\neq i}k(\boldsymbol{y}_i,\boldsymbol{y}_j) \nonumber \\
    &\quad\quad-\frac{2}{n_*^2}\sum_{i=1}^n\sum_{j=1}^nk(\boldsymbol{x}_i,\boldsymbol{y}_i).
    \label{eq:MMD_est}
\end{align}

\begin{lem} \label{lem:MMD}
    Using the batched distance algorithm in Algorithm \ref{alg:BD}, the runtime complexity of using the maximum mean discrepancy (MMD) drops from $O(Bmn^2k^2)$ to $O(nmk^2)$, where $B$ is the number of iterations in permutation testing.
\end{lem}

\begin{proof}
    It is known that for 2 samples of size $N$, MMD can be computed in $d(N,m)=O(mN^2)$ time -- this can be derived by observing \eqref{eq:MMD_est}. 
    \begin{itemize}
        \item Without batching, computing the p-value for MMD via permutation testing would take $O(BmN^2)=O(Bmn^2k^2)$, where we set $N=nk$.
        \item For the BD algorithm, setting $d(k,m)=O(mk^2)$ in Lemma \ref{lem:BD_runtime} gives us the runtime complexity.
    \end{itemize}
\end{proof}

\textbf{Kullback-Leibler divergence.} The KL divergence between two distributions $P$ and $Q$ is defined as
\begin{align}
    D_{KL}(P\|Q)=\E_{X\sim P}\Big[-\log\frac{P(X)}{Q(X)}\Big].
\end{align}
KL divergence can be computed by using $k$-nearest-neighbour density estimation as an intermediate step \cite{Perez2008}.

\begin{lem} \label{lem:KL}
    Using the batched distance algorithm in Algorithm \ref{alg:BD}, the runtime complexity of using the maximum mean discrepancy (MMD) drops from $O(Bmn^2k^2)$ to $O(mnk^2)$, where $B$ is the number of iterations in permutation testing.
\end{lem}

\begin{proof}
    It is known that for 2 samples of size $N$, KL divergence can be computed in $d(N,m)=O(mN^2)$ time \cite{Perez2008}. 
    \begin{itemize}
        \item Without batching, computing the p-value for MMD via permutation testing would take $O(BmN^2)=O(Bmn^2k^2)$, where we set $N=nk$.
        \item For the BD algorithm, we have $d(k,m)=O(mk^2)$ in Lemma \ref{lem:BD_runtime} gives us the runtime complexity.
    \end{itemize}
\end{proof}

\section{Experiments for Problem 1}

The code for all the experiments in this section is presented in \cite{Tan2025}.

\subsection{Comparing batch number and size} 

We study how varying $k/n$ affects the performance of the batched distance algorithm. This study this via numerical simulations instead of analyzing the statistical power analytically because the analytical approach is not feasible -- namely it is difficult to characterize the formula for the variance of estimated distances (and this variance is required for a formula-type analysis of the power). 

We fix the total size $nk=5000$ and vary the ratio
\begin{align*}
    \frac{k}{n}
    =\Big\{\frac{5}{1000},\frac{25}{200},\frac{50}{100},\frac{100}{50},\frac{200}{25},\frac{250}{20},\frac{500}{10}\Big\}.
\end{align*}
For each ratio, we run 100 simulations to calculate the FPR and FNR. We set $\alpha=0.05$, $m=100$, and sample the training set and reference set i.i.d.~from $\normal(\boldsymbol{0}_m,\boldsymbol{I}_m)$. We have chosen drifts that are difficult for the algorithm to detect, so that the changes will be more pronounced. We look at a few situations:
\begin{enumerate}
    \item \textbf{No drift.} The detection window is sampled i.i.d.~from $\normal(\boldsymbol{0}_m,\boldsymbol{I}_m)$. Results are shown in Figure \ref{fig:size_ratio}(a) -- there is no clear relationship between FPR and $k/n$.
    \item \textbf{Mean drift in all feature dimensions.} The detection window is sampled i.i.d.~from $\normal(0.03\cdot\boldsymbol{1}_m,\boldsymbol{I}_m)$. Results are shown in Figure \ref{fig:size_ratio}(b) -- for MMD, FNR decreases and $k/n$ increases. There is no clear relationship between FNR and $k/n$ for EMD and KL.
    \item \textbf{Variance drift in all feature dimensions.} The detection window is sampled i.i.d.~from $\normal(\boldsymbol{1}_m,1.01\cdot\boldsymbol{I}_m)$. Results are shown in Figure \ref{fig:size_ratio}(c) -- for both EMD and KL, as $k/n$ is increased, FNR first drops and then increases. There is no clear relationship between FNR and $k/n$ for MMD.
    \item \textbf{Covariance drift in all feature dimensions.} The detection window is sampled i.i.d.~from $\normal(\boldsymbol{1}_m,\boldsymbol{\Sigma}_m)$. The diagonal entries of $\boldsymbol{\Sigma}_m$ are all ones, and the off-diagonals are all $0.07$. Results are shown in Figure \ref{fig:size_ratio}(d) -- for MMD, FNR decreases and $k/n$ increases. There is no clear relationship between FNR and $k/n$ for EMD and KL.
\end{enumerate}
Overall, MMD works better with a higher $k/n$ whereas EMD and KL work better with a lower $k/n$ (around $0.5$ to $2$). We also point out that increasing $k/n$ makes the BD algorithm run slower, and therefore recommend $k/n$ to be around $1$ to $2$.

\begin{figure}[t]
    \centering
    \includegraphics[width=1.0\columnwidth]{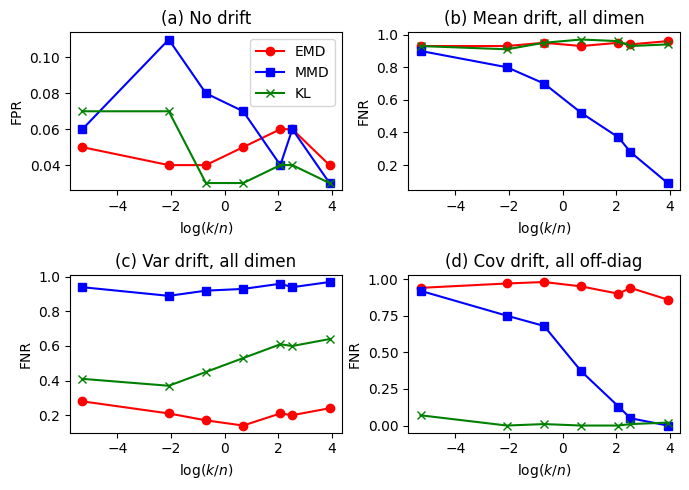}
    \caption{Plots of how the ratio $k/n$ affects the BD algorithm.}
    \label{fig:size_ratio}
\end{figure}

\subsection{Comparing with other approaches} \label{sec:comparing_approaches}

We consider EMD, MMD, and KL divergence, and compare their performances on synthetic datasets with other approaches, namely Kolmogorov–Smirnov test with Bonferroni correction (KS-BC), and permutation testing. Note that EMD, MMD, and KL divergence are estimated using the algorithms that we referenced in Lemma \ref{lem:EMD}, Lemma \ref{lem:MMD}, and Lemma \ref{lem:KL}. We introduce other approaches that allow us to control the FPR:
\begin{enumerate}
    \item \textbf{Kolmogorov–Smirnov test with Bonferroni correction \cite[Chapter 13.3]{James2013}.} This approach uses multiple univariate testing, with a test for each dimension. For each dimension $i$, we test, using the Kolmogorov–Smirnov test, if the dataset from the $i$th dimension of $\mc{S}_T\cup\mc{S}_R$ is equivalent to the dataset from the $i$th dimension of $\mc{S}_D$. If the p-value at any dimension is less than the significant level $\alpha$, then the algorithms detects a drift. We have chosen Bonferroni correction due to its simplicity and conservative nature -- there are other correction methods \cite[Chapter 13]{James2013} that an analyst could consider.
    \item \textbf{Permutation testing \cite[Chapter 13.5]{James2013}.} This approach uses multivariate testing on all dimensions. This approach computes the chosen distance (e.g., EMD, MMD, KL) between $\mc{S}_T\cup\mc{S}_R$ and $\mc{S}_D$. The p-value is then obtained by carrying out a permutation test with $B=100$ permutations -- see Algorithm \ref{alg:perm_test}.
\end{enumerate}

\begin{algorithm}
    \caption{Permutation testing} \label{alg:perm_test}
    \begin{algorithmic}
        \STATE \textbf{input:} training set $\mathcal{S}_T$, reference set $\mathcal{S}_R$, detection set $\mathcal{S}_D$, significant level $\alpha$
        \STATE Compute a distance metric $d=\text{dist}(\mc{S}_T\cup\mc{S}_R,\mc{S}_D)$.
        \FOR{iteration $b\in\{1,\dots,B\}$}
            \STATE Permuate $\mc{S}_T\cup\mc{S}_R\cup\mc{S}_D$ at random. Call the first $2nk$ permuted observations $\mc{S}_T^*\cup\mc{S}_R^*$ and call the remaining $nk$ observations $\mc{S}_D^*$.
            \STATE Compute $d_b^*=\text{dist}(\mc{S}_T^*\cup\mc{S}_R^*,\mc{S}_D^*)$.
        \ENDFOR
        \STATE The p-value is given by $\frac{1}{B}\sum_{b=1}^B\mathds{1}\{|d_b^*|\geq|d|\}$.
        \STATE Drift is detected if the p-value is below $\alpha$.
    \end{algorithmic}
\end{algorithm}

\textbf{Runtime of KS-BC.} For completeness, we provide the runtime complexity of Kolmogorov–Smirnov test with Bonferroni correction (KS + BC), which is $O(nk\log(nk))$. To see why, note that computing the KS test statistic for data of size $O(nk)$ requires sorting all $O(nk)$ observations and therefore takes $O(nk\log(nk))$ time \cite{Gonzalez1977}. Since we have $m$ dimensions, the total runtime is $O(mnk\log(nk))$.

\textbf{Experiment details.} We study the following approaches: (i) batch distance algorithm with EMD, MMD, and KL (denoted as EMD-BD, MMD-BD, and KL-BD for short), (ii) permuation testing with EMD, MMD, and KL (denoted as EMD-PT, MMD-PT, and KL-PT for short), and (iii) KS-BC. Recall that we our focus is the big data regime, where we have limited time (i.e., computational resource) but large volumes of data. Therefore, in our experiments, we constrain the total number of iteration for each approach to around $10^8$, and restrict the number of feature dimensions to $m=100$ and the number of permutations to $B=100$. The choices of the remaining parameters are shown in Table \ref{tab:parameter_sizes}.

\begin{table}[h]
    \caption{Parameter choices for the different approaches such that the total number of iterations is around $10^8$.}
    \begin{center}
    \begin{tabular}{|c|c|c|}
    \hline
    \textbf{Approach} & \textbf{Complexity} & \textbf{Parameters} \\
    \hline
    EMD-BD & $O(n(k^3+mk^2))$ & $n=50,k=100$ \\
    EMD-PT & $O(B(n^3k^3+mn^2k^2))$ & $nk=76$ \\
    MMD-BD & $O(mnk^2)$ & $n=100,k=100$ \\
    MMD-PT & $O(Bmn^2k^2)$ & $nk=100$ \\
    KL-BD & $O(mnk^2)$ & $n=100,k=100$ \\
    KL-PT & $O(Bmn^2k^2)$ & $nk=100$ \\
    KS-BC & $O(mnk\log(nk))$ & $nk=87,900$ \\
    \hline
    \multicolumn{3}{l}{}
    \end{tabular}
    \label{tab:parameter_sizes}
    \end{center}
\end{table}

We run $100$ simulations to calculate the false positive rate (FPR) and the false negative rate (FNR) for every case considered. In each simulation, we use $\alpha=0.05$. Furthermore, the training set and reference set are sampled i.i.d.~from $\normal(\boldsymbol{0}_m,\boldsymbol{I}_m)$. We look at a few scenarios:
\begin{enumerate}
    \item \textbf{No drift.} The detection window is sampled i.i.d.~from $\normal(\boldsymbol{0}_m,\boldsymbol{I}_m)$. The FNR will always be 0, therefore we will only compute the FPR.
    \item \textbf{Mean drift in all feature dimensions.} The detection window is sampled i.i.d.~from $\normal(\zeta\cdot\boldsymbol{1}_m,\boldsymbol{I}_m)$. We vary $\zeta$ from the set $\{0.01,0.02,0.03,0.04\}$. The FPR will always be 0, therefore we will only compute the FNR.
    \item \textbf{Variance drift in all feature dimensions.} The detection window is sampled i.i.d.~from $\normal(\boldsymbol{1}_m,\zeta\cdot\boldsymbol{I}_m)$. We vary $\zeta$ from the set $\{1.005,1.01,1.05,1.10\}$. The FPR will always be 0, therefore we will only compute the FNR.
    \item \textbf{Covariance drift in all feature dimensions.} The detection window is sampled i.i.d.~from $\normal(\boldsymbol{1}_m,\boldsymbol{\Sigma}_m)$. The diagonal entries of $\boldsymbol{\Sigma}_m$ are all ones, and the off-diagonals are all $\zeta$. We vary $\zeta$ from the set $\{0.05, 0.06, 0.07, 0.08\}$. The FPR will always be 0, therefore we will only compute the FNR.
\end{enumerate}

\begin{table*}[h]
    \caption{Experiment results for various methods for drift detection without labels.}
    \begin{center}
    \begin{tabular}{|c|c|cccc|cccc|cccc|}
    \hline
    \textbf{Approach} & \textbf{No drift (FPR)} 
        & \multicolumn{4}{c|}{\textbf{Mean drift with varying $\zeta$ (FNR)}}
        & \multicolumn{4}{c|}{\textbf{Var drift with varying $\zeta$ (FNR)}}
        & \multicolumn{4}{c|}{\textbf{Cov drift with varying $\zeta$ (FNR)}} \\
     &  & $0.01$ & $0.02$ & $0.03$ & $0.04$
        & $1.005$ & $1.01$ & $1.05$ & $1.10$
        & $0.05$ & $0.06$ & $0.07$ & $0.08$ \\
    \hline
    EMD-BD & 0.04 & 0.96 & 0.96 & 0.91 & 0.89 & 0.7 & 0.16 & 0 & 0 & 0.93 & 0.95 & 0.9 & 0.91 \\
    EMD-PT & 0.04 & 0.96 & 0.96 & 0.93 & 0.95 & 0.95 & 0.93 & 0.4 & 0.09 & 0.76 & 0.78 & 0.61 & 0.53 \\
    MMD-BD & 0.03 & 0.96 & 0.85 & 0.25 & 0.01 & 0.97 & 0.96 & 0.83 & 0.08 & 0.63 & 0.4 & 0.03 & 0.03 \\
    MMD-PT & 0.04 & 0.94 & 0.93 & 0.96 & 0.84 & 0.95 & 0.95 & 0.93 & 0.89 & 0.92 & 0.92 & 0.86 & 0.89 \\
    KL-BD  & 0.06 & 0.94 & 0.95 & 0.98 & 0.91 & 0.69 & 0.16 & 0 & 0 & 0.05 & 0 & 0 & 0 \\
    KL-PT  & 0.04 & 0.92 & 0.93 & 0.95 & 0.93 & 0.93 & 0.9 & 0.5 & 0.04 & 0.97 & 0.99 & 0.99 & 1 \\
    KS-BC  & 0.05 & 0 & 0 & 0 & 0 & 0.98 & 0.95 & 0 & 0 & 0.91 & 0.93 & 0.98 & 0.99 \\
    \hline
    \end{tabular}
    \label{tab:experiment_drift_detect_wo_labels}
    \end{center}
\end{table*}

The experiment results are displayed in Table \ref{tab:experiment_drift_detect_wo_labels}. To supplement the numerical results, we also provide visual plots for the results of mean drift in all feature dimensions (scenario 2) in Figure \ref{fig:non_classifier_comparison}(a), the results of variance drift in all feature dimensions (scenario 3) in Figure \ref{fig:non_classifier_comparison}(b), and the results of Variance drift in all feature dimensions (scenario 4) in Figure \ref{fig:non_classifier_comparison}(c).
 We make the following observations:
\begin{enumerate}
    \item \textbf{No drift.} From Table \ref{tab:experiment_drift_detect_wo_labels}, we above that the FPRs are around $0.05$, indicating that it is well controlled. 
    \item \textbf{Mean drift in all feature dimensions.} From Table \ref{tab:experiment_drift_detect_wo_labels} and Figure \ref{fig:non_classifier_comparison}(a), we observe that KS-BC and MMD-BD both performed well, with MMD-BD significantly outperforming MMD-PT. The other approaches are less powerful, performing similarly.
    \item \textbf{Variance drift in all feature dimensions.} From Table \ref{tab:experiment_drift_detect_wo_labels} and Figure \ref{fig:non_classifier_comparison}(b), we observe that EMD-BD and KL-BD performed the best, with EMD-BD, MMD-BD, and KL-BD outperforming EMD-PT, MMD-PT, and KL-PT.
    \item \textbf{Covariance drift in all feature dimensions.} From Table \ref{tab:experiment_drift_detect_wo_labels} and Figure \ref{fig:non_classifier_comparison}(c), we observe that KL-BD performed the best, with KL-BD and MMD-BD significantly outperforming KL-PT and MMD-PT. In this setting EMD-PT outperformed EMD-BD. Also, as expected, KS-BC is unable to detect drifts between features.
\end{enumerate}

\textbf{Overall observation.} Each distance performed well for a different type of drift, and this implies the following recommendations: If mean drift is suspected, then using MMD-BD or KS-BC might be a viable option. If variance drift is suspected, then using EMD-BD or KL-BD might be a viable option. If covariance drift is suspected, then using KL-BD might be a viable option. Overall, whenever a distance is effective at detecting a certain drift (i.e., mean, variance, or covariance), then using that particular distance with the batched distance algorithm outperforms permutation testing, validating the superiority of the batched distance algorithm.

\textbf{Why batching helps.} When we have a runtime budget, batching leads to better statistical power as compared to permutation testing. This is because in order to meet the runtime budget, permutation testing has to reduce the size of its input datasets. The BD algorithm, being more efficient, can use a significantly larger input dataset to meet the same runtime budget. The `gain' in dataset size for the BD algorithm leads to a better statistical power.

Since each distance has its own strengths, one might wonder whether we could do better by combining the outputs of different approaches. This spurs us in the direction of combining the outputs of several approaches to conduct drift detection, which leads us into the next part of the paper.
 
\begin{figure}[t]
    \centering
    \includegraphics[width=1.0\columnwidth]{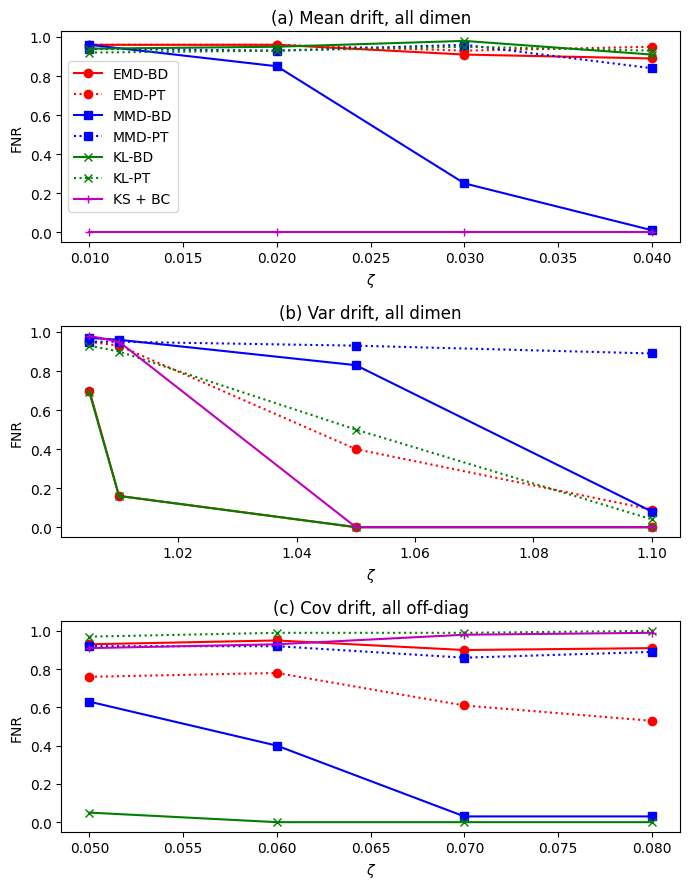}
    \caption{Plots comparing the different approaches.}
    \label{fig:non_classifier_comparison}
\end{figure}

\section{Algorithms for Problem 2}

\begin{algorithm}
    \caption{Classifier algorithm} \label{alg:classifier}
    \begin{algorithmic}
        \STATE \textbf{input:} training set $\mathcal{S}_T$, reference set $\mathcal{S}_R$, detection set $\mathcal{S}_D$, prior detections $\big\{\mc{S}_T^{(i)},\mc{S}_R^{(i)},\mc{S}_D^{(i)},z^{(i)}\big\}_{i=1}^M$ where $i\in\{1,\dots,M\}$, batch number $n$ and batch size $k$ for BD algorithm, choice of classifier, and classifier threshold $\xi$.
        \STATE Initialize the classifier training set $\mc{C}$ and output type (p-values or test statistics).
        \FOR{$i\in\{1,\dots,M\}$}
            \STATE Compute outputs $p_1^{(i)},p_2^{(i)},p_3^{(i)}$ using BD algorithm with EMD (for $p_1^{(i)}$), MMD (for $p_2^{(i)}$), KL (for $p_3^{(i)}$) on $\{\mc{S}_T^{(i)},\mc{S}_R^{(i)},\mc{S}_D^{(i)}\}$.
            \STATE Compute output $p_4^{(i)}$ using KS + BC on $\{\mc{S}_T^{(i)},\mc{S}_R^{(i)},\mc{S}_D^{(i)}\}$, where we output the minimum p-values amongst all the dimensions.
            \STATE Add $\{p_1^{(i)},p_2^{(i)},p_3^{(i)},p_4^{(i)},z^{(i)}\}$ to $\mc{C}$.
        \ENDFOR
        \STATE Standardize $\mc{C}$ to give $\widetilde{\mc{C}}$.
        \STATE Train the chosen classifier on $\widetilde{\mc{C}}$.
        \STATE Compute outputs (p-values or test statistics) $p_1,p_2,p_3$ using BD algorithm with EMD (for $p_1$), MMD (for $p_2$), KL (for $p_3$) on $\{\mc{S}_T,\mc{S}_R,\mc{S}_D\}$.
        \STATE Compute output (p-value or test statistics) $p_4$ using KS + BC on $\{\mc{S}_T,\mc{S}_R,\mc{S}_D\}$, where we choose the minimum output amongst all the dimensions.
        \STATE Standardize $\{p_1,p_2,p_3,p_4\}$ and feed it into the trained classifier to obtain the classifier probability $p_c$.
        \STATE Drift is detected if $p_c\leq\xi$.
    \end{algorithmic}
\end{algorithm}

Recall that under this setting, we have access to prior detections (see Section \ref{sec:setup}). We do not use permutation testing to generate the p-values as they are impractical in the big data regime, especially in this setting (problem 2), since we are dealing with significantly more datasets (i.e., training, reference, and detection sets) compared to the previous setting (problem 1). We first consider two ways to threshold the combined p-values:
\begin{itemize}
    \item \textbf{Taking average.} The most obvious and naive approach is to take the average of the p-values of the different approaches. We then threshold the average p-value to make the decision, i.e., declare drift if the average p-value is below the significant level $\alpha$, and declare no drift otherwise.
    \item \textbf{Perceptron learning} A slightly more sophisticated way is to use perceptron learning to learn the linear combination of the p-values before thresholding. Specifically, the perceptron algorithm is used to learn the weights $\boldsymbol{w}\in\mb{R}^4$ from the p-values $\boldsymbol{p}\in\mb{R}^4$ (we have 4 p-values) such that
    \begin{align*}
        z=\mathds{1}\{\boldsymbol{w}^\top\boldsymbol{p}+\alpha>0\}
        =\begin{cases}
            1 &\text{if $-\boldsymbol{w}^\top\boldsymbol{p}<\alpha$} \\
            0 &\text{if $-\boldsymbol{w}^\top\boldsymbol{p}\geq\alpha$}.
        \end{cases}
    \end{align*}
    Note that $\alpha$ here is fixed, and would not be learnt by the algorithm.
\end{itemize}
A more sophisticated way is to consider the general classifier algorithm presented in Algorithm \ref{alg:classifier}, where we can choose whether to use p-values or test statistics as features. We do not use permutation testing to generate the outputs as they are inpractical in the big data regime. The main question is to determine if we can do better given that we have knowledge about prior detections and are allowed to combine the outputs of different approaches. We use the following choice of classifiers:
\begin{itemize}
    \item \textbf{Logistic regression (LR).} The algorithm predicts the probability of a binary outcome using a linear combination of input features passed through a sigmoid function. The probability of a class can be predicted by calculating the sigmoid of a weighted sum of the input features.
    \item \textbf{K-nearest neighbour (kNN).} Classifies or predicts data points based on the majority class or average value of its 'k' nearest neighbors in the feature space. The probability of a class can be predicted by the proportion of neighbors belonging to that class.
    \item \textbf{Multi-layer perceptron (MLP).} A feedforward neural network with multiple layers that learns complex patterns by adjusting the weights of connections between neurons. The probability of a class can be predicted by applying the softmax function on the output layer's activations.
\end{itemize}
The FPR can be controlled by adjusting the probability threshold used in each of the algorithms presented above.

\section{Experiments for Problem 2}


The code for all the experiments in this section is presented in \cite{Tan2025}. We first generate our training set and then generate our testing set before testing our classifier on the testing set. 

\textbf{Training data generation.} We set $n=50$, $k=100$, $m=100$, and the training set and reference set are sampled i.i.d.~from $\normal(\boldsymbol{0}_m,\boldsymbol{I}_m)$. We look at the following scenarios (which are exactly the same as those used in Section \ref{sec:comparing_approaches}):
\begin{enumerate}
    \item \textbf{No drift.} The detection set is sampled i.i.d.~from $\normal(\boldsymbol{0}_m,\boldsymbol{I}_m)$.
    \item \textbf{Mean drift in all feature dimensions.} The detection window is sampled i.i.d.~from $\normal(\zeta\cdot\boldsymbol{1}_m,\boldsymbol{I}_m)$. We vary $\zeta$ from the set $\{0.01,0.02,0.03,0.04\}$.
    \item \textbf{Variance drift in all feature dimensions.} The detection window is sampled i.i.d.~from $\normal(\boldsymbol{1}_m,\zeta\cdot\boldsymbol{I}_m)$. We vary $\zeta$ from the set $\{1.005,1.01,1.05,1.10\}$.
    \item \textbf{Covariance drift in all feature dimensions.} The detection window is sampled i.i.d.~from $\normal(\boldsymbol{1}_m,\boldsymbol{\Sigma}_m)$. The diagonal entries of $\boldsymbol{\Sigma}_m$ are all ones, and the off-diagonals are all $\zeta$. We vary $\zeta$ from the set $\{0.05, 0.06, 0.07, 0.08\}$.
\end{enumerate}

\textbf{Training data generation.} We sample 50 sets of $\{\mc{S}_T^{(i)},\mc{S}_R^{(i)},\mc{S}_D^{(i)}\}$ from scenario 1 with $z^{(i)}=0$, 40 sets of $\{\mc{S}_T^{(i)},\mc{S}_R^{(i)},\mc{S}_D^{(i)}\}$ from scenario 2 with $z^{(i)}=1$ with 10 sets for each $\zeta$, 40 sets of $\{\mc{S}_T^{(i)},\mc{S}_R^{(i)},\mc{S}_D^{(i)}\}$ from scenario 3 with $z^{(i)}=1$ with 10 sets for each $\zeta$, 40 sets of $\{\mc{S}_T^{(i)},\mc{S}_R^{(i)},\mc{S}_D^{(i)}\}$ from scenario 4 with $z^{(i)}=1$ with 10 sets for each $\zeta$.

\textbf{Testing data generation.} We run 100 simulations for each scenario to calculate the FPR and the FNR for every case considered -- the cases are exactly the same as those presented in the training data generation (see above). In each simulation, we use $n=50$, $k=100$, and $\alpha=0.05$.

\textbf{Classifier details.} We use the following parameters for each of our classifiers.
\begin{itemize}
    \item \textbf{Logistic regression (LR).} We use threshold $\xi=0.8$ in Algorithm \ref{alg:classifier}. 
    \item \textbf{K-nearest neighbour (kNN).} We use threshold $\xi=0.85$ in Algorithm \ref{alg:classifier}, and set the number of nearest neighbour to $10$.
    \item \textbf{Multi-layer perceptron (MLP).} We use threshold $\xi=0.8$ in Algorithm \ref{alg:classifier}, 2 hidden layers of sizes $4$ and $2$, and the ReLU activation function.
\end{itemize}

The thresolds $\xi$ are chosen to control the FPR to be around $0.05$. We use the following shorthands in our experiment results:
\begin{itemize}
    \item AVG for the average p-value approach, and PL for perceptron learning. 
    \item LR-p, kNN-p, MLP-p for logistic regression, k-nearest neighbour, and multi-layer perceptron, respectively, in Algorithm \ref{alg:classifier} using p-values as features.
    \item LR-s, kNN-s, MLP-s for logistic regression, k-nearest neighbour, and multi-layer perceptron, respectively, in Algorithm \ref{alg:classifier} using test statistics as features. 
\end{itemize}
\textbf{FPR and FNR.} The experimental results for FPR and FNR are shown in Table \ref{tab:experiment_drift_detect_w_prior_labels}. To supplement the numerical results, we provide visual plots for the classifier approach in Figure \ref{fig:classifier_comparison}. We make the following observations.
\begin{enumerate}
    \item \textbf{No drift.} As expected, the FPR is kept low around 0.05, since we have adjusted the classifier's predicted probability threshold so that the FPR is comparable to those of previous experiments in problem 1.
    \item \textbf{Mean drift in all feature dimensions.} As expected, the more sophisticated approaches LR-s and MLP-s, along with KS-BC, performed the best. This shows that KS-BC is very suitable for mean drift detection and we cannot do significantly better with prior detection information. 
    \item \textbf{Variance drift in all feature dimensions.} Overall, the most sophisticated approach, MLP-p, performed the best.
    \item \textbf{Covariance drift, all off-diagonals.} Overall, the most sophisticated approach, MLP-p, performed the best.
\end{enumerate}

\begin{table*}[h]
    \caption{Experiment results for various methods for drift detection with prior detections.}
    \begin{center}
    \begin{tabular}{|c|c|cccc|cccc|cccc|}
    \hline
    \textbf{Approach} & \textbf{No drift (FPR)} 
        & \multicolumn{4}{c|}{\textbf{Mean drift with varying $\zeta$ (FNR)}}
        & \multicolumn{4}{c|}{\textbf{Var drift with varying $\zeta$ (FNR)}}
        & \multicolumn{4}{c|}{\textbf{Cov drift with varying $\zeta$ (FNR)}} \\
     &  & $0.01$ & $0.02$ & $0.03$ & $0.04$
        & $1.005$ & $1.01$ & $1.05$ & $1.10$
        & $0.05$ & $0.06$ & $0.07$ & $0.08$ \\
    \hline
    AVG    & 0    & 1 & 1 & 1 & 0.75 & 1 & 1 & 1 & 0 & 0.75 & 1 & 1 & 0.5 \\
    PL     & 0.02 & 0.75 & 1 & 1 & 0.75 & 0.75 & 0.75 & 1 & 0 & 0.5 & 0 & 0.25 & 0 \\
    LR-p   & 0.09 & 0.82 & 0.62 & 0.41 & 0.21 & 0.8 & 0.62 & 0.29 & 0.04 & 0.36 & 0.26 & 0.05 & 0.11 \\
    kNN-p  & 0.06 & 0.89 & 0.73 & 0.61 & 0.48 & 0.88 & 0.64 & 0.25 & 0.04 & 0.47 & 0.31 & 0.14 & 0.14 \\
    MLP-p  & 0.09 & 0.76 & 0.57 & 0.46 & 0.35 & 0.68 & 0.34 & 0.06 & 0.02 & 0.24 & 0.12 & 0.01 & 0.02 \\
    LR-s   & 0.05 & 0.87 & 0.45 & 0.1 & 0 & 0.92 & 0.88 & 0.06 & 0 & 0.52 & 0.27 & 0.06 & 0.01 \\
    kNN-s  & 0.07 & 0.91 & 0.65 & 0.2 & 0.02 & 0.94 & 0.91 & 0.01 & 0 & 0.57 & 0.28 & 0.06 & 0.02 \\
    MLP-s  & 0.06 & 0.88 & 0.51 & 0.11 & 0 & 0.95 & 0.97 & 0.33 & 0 & 0.67 & 0.36 & 0.12 & 0.05 \\
    EMD-BD & 0.07 & 0.93 & 0.93 & 0.93 & 0.97 & 0.69 & 0.16 & 0 & 0 & 0.94 & 0.94 & 0.95 & 0.94 \\
    MMD-BD & 0.07 & 0.92 & 0.88 & 0.52 & 0.17 & 0.93 & 0.93 & 0.85 & 0.4 & 0.79 & 0.55 & 0.37 & 0.23 \\
    KL-BD  & 0.04 & 0.96 & 0.98 & 0.97 & 0.93 & 0.82 & 0.46 & 0 & 0 & 0.3 & 0.06 & 0 & 0 \\
    KS-BC  & 0.06 & 0.91 & 0.48 & 0.13 & 0 & 0.94 & 0.98 & 0.88 & 0.53 & 0.95 & 0.95 & 0.95 & 0.95 \\
    \hline
    \end{tabular}
    \label{tab:experiment_drift_detect_w_prior_labels}
    \end{center}
\end{table*}

\begin{figure}[t]
    \centering
    \includegraphics[width=1.0\columnwidth]{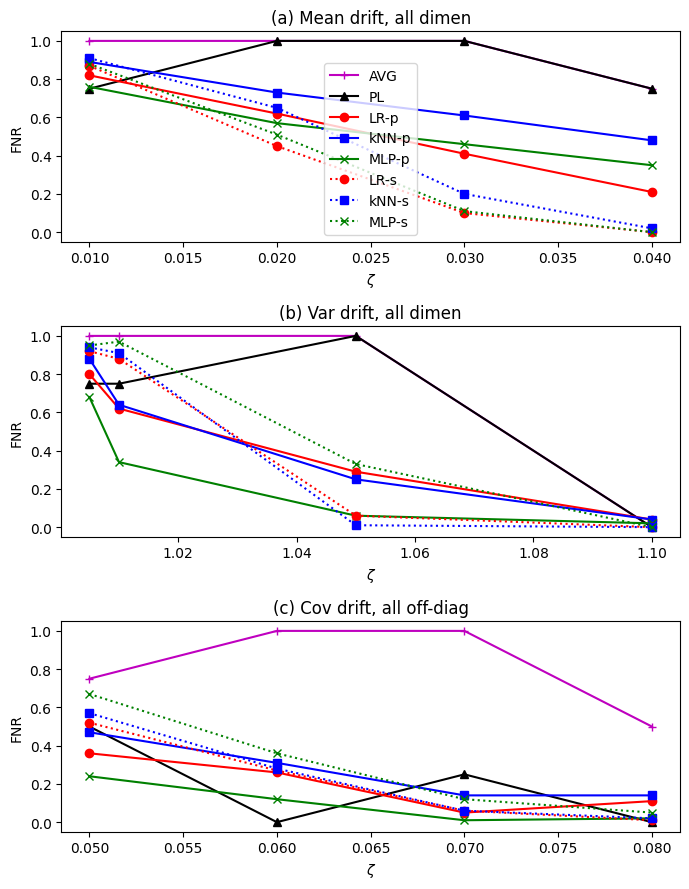}
    \caption{Plots comparing the different approaches that utilize prior detection information.}
    \label{fig:classifier_comparison}
\end{figure}

\textbf{Accuracy.} The accuracy results for the various detection approaches are as follows are presented in Table \ref{tab:classification_accuracy}. We observe that MLP-p is the best when p-values are used as features, and LR-s is the best when test statistics are used as features.

\begin{table}[h!]
    \caption{Accuracy results for the various detection approaches.}
    \begin{center}
    \begin{tabular}{|l|c|}
    \hline
    \textbf{Approach} & \textbf{Accuracy} \\
    \hline
    AVG     & 23.1\% \\
    PL      & 47.9\% \\
    LR-p    & 64.0\% \\
    kNN-p   & 56.6\% \\
    MLP-p   & 71.4\% \\
    LR-s    & 67.8\% \\
    kNN-s   & 64.3\% \\
    MLP-s   & 61.5\% \\
    EMD-BD  & 35.0\% \\
    MMD-BD  & 41.5\% \\
    KL-BD   & 57.5\% \\
    KS-BC   & 33.0\% \\
    \hline
    \end{tabular}
    \label{tab:classification_accuracy}
    \end{center}
\end{table}

As expected, approaches that use prior detection information generally perform better than those that do not use prior detection information. Furthermore, the most sophisticated classifier (MLP) performed the best in terms of accuracy.

\textbf{Remark.} We admit that the classifier-based fusion of p-values and test statistics may be prone to overfitting. Hence, it would be interesting to design better approaches for this semi-supervised setting -- we leave this as a future direction.

\section{Conclusion}

We considered drift detection without labels in the big data regime where there is a huge volume of data, but there is a constraint on the computational resources (i.e., runtime budget). In this regime, we introduce the batched distance algorithm and show experimentally that it outperforms permutation testing in terms of statistical power, while also keeping the false positive rate low. Furthermore, we introduced a new semi-supervised drift detection framework to model the scenario of detecting drift (without labels) given prior detections. Under this setting we showed that our drift detection algorithm can be incorporated effectively into this framework. 





\bibliographystyle{IEEEtran}
\bibliography{IEEEabrv,references}

\end{document}